\documentclass{article}
\usepackage{geometry}
\usepackage[utf8]{inputenc} 
\usepackage[T1]{fontenc}    
\usepackage{hyperref}       
\usepackage{url}            
\usepackage{booktabs}       
\usepackage{amsfonts}       
\usepackage{nicefrac}       
\usepackage{microtype}      
\usepackage{xcolor}         

\usepackage{graphicx}

\usepackage{amsmath, amsthm, amssymb, hyperref, color}
\usepackage{bm}

\usepackage{forest}
\usepackage{tikz}
\usepackage{tikz-qtree}
\usepackage{tikz-3dplot}

\usepackage[export]{adjustbox} 

\newtheorem{theorem}{Theorem}
\newtheorem{proposition}[theorem]{Proposition}
\newtheorem{lemma}[theorem]{Lemma}

\theoremstyle{definition}

\newtheorem{example}[theorem]{Example}

\newcommand{\Z}{\mathcal{Z}}
\newcommand{\X}{\mathcal{X}}
\newcommand{\Y}{\mathcal{Y}}
\def\ci{\perp\!\!\!\perp} 
\newcommand{\eg}{{\em e.g.}}
\newcommand{\ie}{{\em i.e.}}

\usepackage{amsmath,amsfonts,bm}

\def\1{\bm{1}}

\def\vu{{\bm{u}}}
\def\vv{{\bm{v}}}
\def\vw{{\bm{w}}}

\DeclareMathAlphabet{\mathsfit}{\encodingdefault}{\sfdefault}{m}{sl}
\SetMathAlphabet{\mathsfit}{bold}{\encodingdefault}{\sfdefault}{bx}{n}

\newcommand{\RR}{\mathbb{R}}

\graphicspath{ {./figures/} }

\title{Compositional Structures in Neural Embeddings\\ and Interaction Decompositions}

\author{Matthew Trager \qquad Alessandro Achille \qquad Pramuditha Perera\\[.3cm] Luca Zancato \qquad Stefano Soatto \\[.6cm] AWS AI Labs}

\begin{document}

\maketitle

\begin{abstract}
We describe a basic correspondence between linear algebraic structures within vector embeddings in artificial neural networks and conditional independence constraints on the probability distributions modeled by these networks. Our framework aims to shed light on the emergence of structural patterns in data representations, a phenomenon widely acknowledged but arguably still lacking a solid formal grounding. Specifically, we introduce a characterization of compositional structures in terms of ``interaction decompositions,'' and we establish necessary and sufficient conditions for the presence of such structures within the representations of a model. 
\end{abstract}
\section{Introduction}

Neural networks today generally operate without imposing explicit statistical
or modeling assumptions on data. For example, transformer
architectures~\cite{vaswaniAttentionAllYou2017} are now used with little or no
architectural modifications for a wide variety of tasks and data types ---
from natural language~\cite{devlinBERTPretrainingDeep2019},
vision~\cite{dosovitskiyImageWorth16x162021}, audio \cite{verma2021audio}, and
multi-modal data \cite{lu2022unified}. The philosophy of deep learning is to
try to learn all relevant structure 'end-to-end' from data, instead of relying
on hand-crafted inductive biases.

This kind of approach has had great empirical success, but the resulting
models generally lack {interpretability} and
{controllability}~\cite{raukerTransparentAISurvey2023}. Indeed, treating a
model's intermediate representations as unstructured makes them difficult to
manipulate and understand by humans. This has motivated research in the area
of ``explainable AI'' (XAI)~\cite{linardatosExplainableAIReview2020}.

Even if models are unstructured, symmetries and patterns in the data may be
reflected inside the model after training. This has been empirically observed
at least since the introduction of word embeddings in
NLP~\cite{mikolovDistributedRepresentationsWords2013}. More recently, emergent
structures have been observed in different
settings~\cite{radfordUnsupervisedRepresentationLearning2016}, sometimes
presented as ``world models''~\cite{li2022emergent,
gurneeLanguageModelsRepresent2023}. However, these findings are almost
entirely empirical, and the emergence of structure --- and particularly of
\emph{linear} structure --- is sometimes still referred to as a
``hypothesis''~\cite{nandaEmergentLinearRepresentations2023}.

The purpose of this note is to spell out a simple but precise mathematical
correspondence between statistical independence conditions on modeled
distributions and geometric patterns within the internal representations of
network models. This correspondence can be clearly stated in terms of
\emph{interaction
decompositions}~\cite{darrochAdditiveMultiplicativeModels1983}. These are
linear algebraic decompositions closely related to the theory of log-linear
expansions and graphical models~\cite{lauritzen1996graphical}. Our main result
extends a fact pointed out for independence models
in~\cite{tragerLinearSpacesMeanings2023}, and shows more generally how
probabilistic structure in the data is reflected in linear algebraic patterns
in the representations. Despite the simplicity of this connection, we are not
aware of other works that present such a general statement in the context of
neural embeddings.

More broadly, we believe that interaction decompositions provide an intuitive
and theoretically grounded framework for describing linear compositional
patterns within neural embeddings, and hold promise as a useful practical tool
for future work on interpretability and controllability.

\section{Related work}
\label{sec:rel-work}

Data representations used in modern machine learning are generally
\emph{distributed}~\cite{hinton1984distributed}, meaning that coordinate
dimensions do not correspond to ``semantic components'' from the data. For
example, in natural language processing (NLP),  \emph{distributional
representations} are constructed by simply observing how words or linguistic
parts co-occur in text~\cite{clarkVectorSpaceModels2015a,
turneyFrequencyMeaningVector2010}. Unlike traditional symbolist approaches,
these representations are not designed to satisfy structural constraints. A
large body of work aims to combine symbolic and distributed elements, \eg,
\cite{smolenskyTensorProductVariable1990, mitchell2008vector, baroni2010nouns,
coeckeMathematicalFoundationsCompositional2010a} (see also
\cite{ferroneSymbolicDistributedDistributional2020} for a survey).

It is now well-established that semantically meaningful structures can
\emph{emerge} within data embeddings, even if they were not explicitly
designed or enforced during training. A classic example is the vector
arithmetic associated with analogies in word
embeddings~\cite{mikolovDistributedRepresentationsWords2013}, but similar
phenomena have been observed in various other
contexts~\cite{radfordUnsupervisedRepresentationLearning2016, li2022emergent,
tragerLinearSpacesMeanings2023}. Emergent structures are also known to
sometimes be \emph{linear}, particularly towards the final layers of the
model~\cite{elhage2022superposition, nandaEmergentLinearRepresentations2023}.
These matters have attracted interest in the area of ``mechanistic
interpretability''~\cite{bricken2022monosemanticity} which aims to
reverse-engineer machine learning models and provide human-understandable
interpretations of their activations. We also
mention~\cite{zouRepresentationEngineeringTopDown2023a} that recently proposed
``representation engineering''  as a top-down investigation of emergent
structures in models for AI transparency.

Despite this widespread recognition, the emergence of latent semantic
structures is often regarded as an empirical phenomenon, and lacks a solid
mathematical grounding. Many theoretical justifications are heavily centered
on word
embeddings~\cite{aroraLatentVariableModel2016,gittensSkipGramZipfUniform2017,ethayarajhUnderstandingLinearWord2019,allenAnalogiesExplainedUnderstanding},
and are not applicable to other models including transformers.  Some recent
exceptions include~\cite{wangConceptAlgebraScoreBased2023}, which considers
linear decompositions of the ``score representation'' and is especially
relevant for diffusion models,
and~\cite{jiangOriginsLinearRepresentations2024a}, which argues that softmax
objective and gradient descent contribute to the emergence of linear
representations (see Section~\ref{sec:main-results} for a comparison with our
results). More similar to our
perspective,~\cite{tragerLinearSpacesMeanings2023} considers general
``linearly decomposable'' vector embeddings and proves that these structures
correspond to independence conditions for probabilistic softmax models. In
this paper, we expand on this observation by: 1) considering a much more
general form of linear structures based on \emph{interaction
decompositions}~\cite{darrochAdditiveMultiplicativeModels1983}, and 2) showing
that these structures precisely correspond to \emph{conditional} independence
relations in softmax models.

Finally, there has been recent discussions on the observed emergence ``world
models'' within transformer-based models~\cite{li2022emergent,
nandaEmergentLinearRepresentations2023, gurneeLanguageModelsRepresent2023},
and whether these structures challenge prior claims that models trained solely
on text lack a fundamental “understanding” of the world
~\cite{benderClimbingNLUMeaning2020}. This work can offer insights into these
debates by explaining in a simple but precise setting when latent geometric
structures can be expected in a model's representations.

\section{Preliminaries}

Our results apply to general ``softmax'' probabilistic models of the form
\begin{equation}\label{eq:distribution} 
P(Y=y|X=x) = \frac{\exp \langle
\vu(x), \vv(y) \rangle}{\sum_{y' \in \mathcal Y} \exp \langle \vu({x}),
\vv({y'})\rangle},
\end{equation} where $\X, \Y$ are sets and $\vu: \X \rightarrow V$ and $\vv:
\Y \rightarrow V$ are ``input'' and ``output'' embeddings into a common vector
space $V$ with a fixed an inner product $\langle \cdot, \cdot \rangle$. For
simplicity, we assume that both $\X$ and $\Y$ are finite; if this is not the
case, we can restrict ourselves without loss of generality to arbitrary finite
subsets.\footnote{The embeddings $\vu, \vv$ determine restrictions
of~\eqref{eq:distribution} to all subsets of the data in a way that is
consistent with marginalization.} Expressions similar to the right-hand side
of~\eqref{eq:distribution} appear as the final processing step for machine
learning models that solve classification-type tasks. For example, in
feed-forward networks $\vu(x)$ is a vector encoding of an input $x$ and
$\vv(y)$ is a row of the final-layer weight matrix corresponding to a
candidate label $y$. Similarly, in auto-regressive language models, $\vu(x)$
is the encoding of a string of text and $\vv(y)$ is the embedding of a
candidate token for the task of next-token prediction. In vision-language
models such as CLIP~\cite{radford2021learning}, $\vu, \vv$ are separate
encodings of an image and of a caption.  In these examples, the embeddings are
defined by parametric models $\vu(x) = \vu(x;\theta)$, and $\vv(y) =
\vv(y;\eta)$ (the latter is often an unconstrained vector). However, once
trained, the parameterization of these embeddings becomes irrelevant. We thus
assume that the model has been trained to accurately represent a ``true'' data
distribution and our goal is to explore how the geometry of embeddings is
related to statistical structures present in the data.

\section{Interaction decompositions}
\label{sec:interaction}

In this section, we introduce interaction decompositions that will be our main
tool for characterizing geometric structures in the embeddings. This sort of
decomposition is well known and discussed for example
in~\cite{lauritzen1996graphical} and~\cite{ayInformationGeometry2017a},
although not in the context of neural embeddings.

Let $\Z = \mathcal Z_1 \times \ldots \times \mathcal Z_k$ be a finite factored
set. For any set of indices $I \subset [k]$, we write $\mathcal Z_I:=\prod_{i
\in I} \mathcal Z_i$ and $z_I:=(z_i)_{i\in I}$. Given  $z = (z_1,\ldots,z_k)
\in \Z$, we refer to each $z_i \in \Z_i$ as a ``variable.'' We consider the
vector space $V^{\mathcal Z} = \{\vw: \mathcal Z \rightarrow V\}$ of all
embeddings. For any $I \subset [k]:=\{1,\ldots,k\}$, we define subspace $E_I
\subset V^{\mathcal Z}$ consisting in $\vw \in V^{\mathcal Z}$ such that:
\begin{enumerate}
    \item $\vw$ depends only on variables in $I$ (\ie, $\vw(z_I, z_{[k] \setminus I})$ is fixed as $z_{[k] \setminus I}$ varies, for all $z_I$);
    \item $\sum_{z_{[k]\setminus J} \in \Z_{[k] \setminus J}} \vw(z_J, z_{[k]\setminus J}) = 0$ for any $J \subsetneq I$, and $z_{J} \in \Z_{J}$.
\end{enumerate}
Intuitively, the space $E_I$ contains embeddings that depend only on the
variables in $I$ but not only on variables of a proper subset $J$ of $I$. We
refer to $E_I$ as the \emph{pure interaction space}. We now have the following
general description of these spaces.

\vspace{.2cm}

\begin{proposition}\label{prop:projections} In the setting described above, there exists a direct sum decomposition of vector spaces
\begin{equation}
V^{\mathcal Z} = \bigoplus_{I \subset [k]} E_I.
\end{equation}
The projection of $V^\Z$ onto $E_I$ is given by
\begin{equation}\label{eq:proj}
Q_I := \sum_{J \subset I} (-1)^{|I \setminus J|} \pi_J,
\end{equation}
where $\pi_J: V^\Z \rightarrow V^\Z$ is described by 
\[
\pi_J(\vw)(z) = \frac{|\Z_J|}{|\Z|}\sum_{z_{[k] \setminus J} \in \Z_{[k] \setminus J}} \vw(z_J, z_{[k] \setminus J}).
\] 
We also have that $\dim(E_I) = \dim(V) \cdot \prod_{i \in I} (|\Z_i| - 1)$.
\end{proposition}

This result implies that any embedding $\vw:\mathcal Z \rightarrow V$ has a \emph{unique} decomposition
\begin{equation}\label{eq:interaction-decomp}
\vw = \sum_{I \subset[k]} \vw_I, \qquad \text{ where } \vw_I \in E_I.
\end{equation}
Here $\vw_I$ depends only on variables in $I$ and thus can also be viewed as a
map $\vw_I: \Z_I \rightarrow V$. We refer to the embeddings $\vw_I$
in~\eqref{eq:interaction-decomp} as the \emph{interaction components} of
$\vw$. Note that all components are easily obtained from $\vw$
via~\eqref{eq:proj}. We sometimes say that $\vw_I$ has \emph{order} equal to
$|I|$. First-order components correspond to factors $\Z_i$, and we can view
$\mathcal V_i := Span(\vw_{i}(z_i), z_i \in \Z_i) = E_{\{i\}} \cap
Span(\vw(z), z \in \Z)$ as the ``factor space'' associated with $\Z_i$ (here
and in the following we write $\vw_i$ instead of $\vw_{\{i\}}$). If $\vw$ only
has components of order $0$ or $1$ then we may write $\vw(z) =
\vw_{\varnothing} + \sum_{i \in [k]} \vw_i(z_i)$, which expresses a
representation as a ``disentangled'' sum of vectors $\vw_i$ associated with
each variable (and each belonging to the corresponding factor space $\mathcal
V_i$), plus a mean vector.

The vanishing of interaction components can be seen as influencing the
geometry of the \emph{polytope} in embedding space whose vertices are
embeddings of elements $z \in \mathcal Z$. Roughly speaking, more vanishing
terms mean that this polytope is lower-dimensional and more ``regular.'' For
example, if all interactions are allowed, then this polytope is generically a
simplex of dimension $\prod_i |\mathcal Z_i|-1$ (spanned by $|\mathcal Z|$
points in general position); on the other extreme, if only unary interaction
terms are allowed, it is a product of simplices corresponding to each factor,
with total (affine) dimension $\sum_i (|\mathcal Z_i| - 1)$. The classical
parallelogram structure for word analogies is a particular example of the
latter situation. See Section~\ref{sec:geometry-examples} in the Appendix for
some visualizations.

Finally, we note that interaction spaces can also be viewed from the
perspective of mathematical representation theory: each space $E_I =
W_I^{\oplus \dim V}$ is associated with an irreducible representation $W_I$ of
a product of symmetric groups $G = \mathfrak S_{|\Z_1|} \times \ldots \times
\mathfrak S_{|\Z_k|}$. This perspective is briefly discussed in the appendix
of~\cite{tragerLinearSpacesMeanings2023}, and is related to the notion of
disentanglement given in~\cite{higgins2018towards}.

\vspace{.3cm}

\section{Main results}
\label{sec:main-results}

We now return to the probabilistic model for $P(Y|X)$ in~\eqref{eq:distribution}. 
We assume that the value sets are factored as set products $\X = \X_1 \times \ldots \times \X_m$ and $\Y = \Y_1 \times \ldots \times \Y_n$. 
This allows us to view each input and output data point as factored into variables $x=(x_1,\ldots, x_m)$, $y=(y_1,\ldots,y_n)$ with $x_i \in \X_i$ and $y_i \in \Y_i$. 
We assume that ``structure'' in the data can be modeled by probabilistic dependencies between these variables. We formalize this as follows.

Given any partition $Z_A, Z_B, Z_C$ of the variables $\{X_1,\ldots,X_m, Y_1,\ldots,Y_n\}$, we say that $Z_A$ and $Z_B$ are conditionally independent given $Z_C$ for $P(Y|X)$
if there {exists} a strictly positive density $p(x)$ over $\X$ such that that $Z_A$ and $Z_B$ are conditionally independent given $Z_C$ for the joint density $q(x,y) = P(Y=y|X=x)p(x)$. We write this condition as $Z_A \ci Z_B \mid_X Z_C$. As before, we write $\Z_A, \Z_B, \Z_C$ for the co-domains of $Z_A, Z_B, Z_C$. This condition can also be understood as follows.

\vspace{.2cm}

\begin{lemma}\label{lemma:cond-ind} A partition of variables satisfies $Z_A \ci Z_B \mid_X Z_C$ for $P(Y|X)$ if and only if there exist functions $f\in \RR^{\Z_A \times \Z_C}$, $g \in \RR^{\Z_B \times \Z_C}$, and a strictly positive $h \in \RR^{\X}$ such that
\begin{equation}\label{eq:factorization-exists}
P(Y=y|X=x) = f(z_A, z_C) g(z_B, z_C)h(x_1,\ldots,x_m),
\end{equation}
where $(z_A, z_B, z_C) = (x,y)$.
\end{lemma}

We now state our main result relating interaction decompositions and conditional independence constraints. In the following, given two sets of indices $I \subset [m]$, $J \subset [n]$, we write $I \sqcup J$ for the disjoint union as a subset of $[m] \sqcup [n] \cong [m+n]$.

\vspace{.3cm}

\begin{theorem}\label{thm:general} Let $Z_A, Z_B, Z_C$ be a partition of the variables $\{X_1,\ldots,X_m, Y_1,\ldots,Y_n\}$. Then for a distribution $P(Y | X)$ as in~\eqref{eq:distribution}, the condition
\begin{equation}\label{eq:ci-a}
Z_A \ci Z_B \mid_X Z_C
\end{equation}
holds \emph{if and only if} the interaction decompositions of the embeddings $\vu = \sum_{I \subset [m]} \vu_I$ and $\vv = \sum_{J \subset [n]} \vv_J$ satisfy
\begin{equation}\label{eq:interaction-vanishing}
\langle \vu_I, \vv_J\rangle = 0, 
\end{equation}
for all $I \subset [m]$ and $J \subset [n]$ with $J \ne \varnothing$ such that $(I \sqcup J) \cap A \ne \varnothing$ and $(I \sqcup J) \cap B \ne \varnothing$.
\end{theorem}

\vspace{.4cm}

Note that the condition on interaction decompositions excludes $J= \varnothing$, which corresponds to the ``mean'' term $\vv_{\varnothing} = \frac{1}{|\mathcal Y|}\vv(y)$: indeed, it is easy to realize that translating the embeddings $\vv$ by any fixed vector does not affect the model~\eqref{eq:distribution}.

This result provides a simple but precise ``dictionary'' for translating probabilistic conditions on the data~\eqref{eq:ci-a} into geometric conditions on the embeddings~\eqref{eq:interaction-vanishing} and vice-versa. While several recent works have proposed similar formalisms for linear structures in embeddings using the softmax loss  (\eg,~\cite{parkLinearRepresentationHypothesis2023,jiangOriginsLinearRepresentations2024a,tragerLinearSpacesMeanings2023}) we believe that our description is simpler while also being more general. For example,~\cite{jiangOriginsLinearRepresentations2024a} consider only pairs of separable (independent) binary concepts, whereas we consider \emph{conditional} independence relations between arbitrary factors. Our treatment also significantly extends that of~\cite{tragerLinearSpacesMeanings2023}, which only considers factorizations of the output variable, while we also deal with factorizations of the input (which is more important in many situations). Another advantage of our description is that it provides \emph{necessary and sufficient conditions} for the existence of geometric structure, as opposed to only sufficient conditions that most prior works focus on. This gives a simple way to impose interpretable probabilistic conditions by manipulating embeddings. As we will discuss, the condition~\eqref{eq:interaction-vanishing} is also significantly more informative than prior intuitions such as ``relations''=``lines''~\cite{aroraLatentVariableModel2016}, or ``features''=``directions''~\cite{elhage2022superposition}.

In the remainder of the section, we describe a few specializations of Theorem~\ref{thm:general} to particular cases of interest, with some slight refinements. First, we assume that $m=1, n>1$, which means that we consider only factorizations of the output variable $y = (y_1,\ldots, y_n)$ and conditional independence on $P_x := P(Y|X=x)$.

\vspace{.2cm}

\begin{proposition}\label{prop:conditional-independence} Let $I,J, K$ be a partition of $[n]$ and let $x \in \X$ be arbitrary. Then the conditional independence relation
\begin{equation}\label{eq:ci}
Y_I \ci Y_J \,\, | \,\, Y_K \,\,\, \mbox{ for } P_x = P(\, \cdot \, | X=x)
\end{equation}
where $P(Y|X)$ is as in~\eqref{eq:distribution}
is equivalent to the interaction components of the embedding $\vv$ satisfying
\begin{equation}\label{eq:interaction-vanishing-2}
\langle \vu(x), \vv_H \rangle = 0,
\end{equation}
for all $H \subset[n]$ such that $H \cap I \ne \varnothing, H \cap J \ne \varnothing$.
In particular, if~\eqref{eq:ci} holds for all $x \in \X_0 \subset \X$ and $Span(\vu(x) \colon x \in \X_0) = V$, then $\vv_H = 0$; this in turn implies that~\eqref{eq:ci} holds for all $x \in \X$.
\end{proposition}

The last part of the statement says that if the conditional independence on output variables holds for a set of inputs $\X_0$ whose representations span $V$, then it must actually must hold for {all} input variables, because certain interaction components actually vanish. For example, if the factors are independent for $P_x$ for all $x$, then only components of order $0$ or $1$ appear in the decomposition of $\vv$ (as in the setting of~\cite[Proposition 7]{tragerLinearSpacesMeanings2023}).

\vspace{.2cm}

We next consider the case $m>1,  n=1$ and $x = (x_1,\ldots,x_m)$. If $I,J,K$ are a partition of $[m]$, then Lemma~\ref{lemma:cond-ind} means that we can write
\begin{equation}\label{eq:relative-causal-independence}
\begin{aligned}
&P(Y=y \, | \, X_I=x_I, X_J=x_J, X_K=x_K)\\
&= f(y,x_I, x_K)g(y, x_J, x_K) h(x_I, x_J, x_K),
\end{aligned}
\end{equation}
for appropriate functions $f, g, h$. This condition is a sort of ``relative causal independence'' between $X_I, X_J$. Indeed, (strict) causal independence can be defined as~\eqref{eq:relative-causal-independence} with $h=1$, since in that case the contribution of $X_I$ and $X_J$ to $P(Y|X)$ can be computed separately. In contrast, when $h \ne 1$, the contributions of $X_I, X_J$ are entangled, but they become independent if we consider ratios of probabilities for $y, y' \in \Y$:
\begin{equation}
\begin{aligned}
&\frac{P(Y=y \, | \, X_I=x_I, X_J=x_J, X_K=x_K)}{P(Y=y' \, | \, X_I=x_I, X_J=x_J, X_K=x_K)} \\
&= \frac{f(y,x_I, x_K)}{f(y',x_I, x_K)} \cdot \frac{g(y, x_J, x_K)}{g(y', x_J, x_K)}.
\end{aligned}
\end{equation}
Since~\eqref{eq:relative-causal-independence} is a weaker constraint than strict causal independence, it is satisfied by more distributions.

\vspace{.4cm}

\begin{proposition}\label{prop:relative-causal-independence} Let $I,J, K$ be a partition of $[m]$. Then the condition of ``relative causal  independence'' in~\eqref{eq:relative-causal-independence} is satisfied for all $y, y' \in \mathcal Y_0 \subset \Y$
if and only if the interaction components of the embedding $\vu$ satisfy
\begin{equation}\label{eq:interaction-vanishing=x}
\langle \vu_H, \vv(y) - \vv(y')\rangle = 0, 
\end{equation}
for all $H \subset [m]$ such that $H \cap I \ne \varnothing, H \cap J \ne \varnothing$ and $y, y' \in \Y_0$.
In particular, if~$Span(\vv(y) - \vv(y') \colon y,y' \in \Y_0) = V$, then $\vu_H = 0$; this in turn implies that~\eqref{eq:relative-causal-independence} holds for $\Y_0 = \Y$.
\end{proposition}

\vspace{.2cm}

We observe that when factorizations of the input variables are considered, the corresponding interaction decompositions are constrained by differences between pairs of output embeddings in~\eqref{eq:interaction-vanishing=x}. As before, this is consistent with the fact that applying translations to $\vv$ does not affect the model.

\vspace{.2cm}

Finally, we assume that $m = n$ and $P(Y|X)$ factors as
\begin{equation}\label{eq:causal-independence}
\begin{aligned}
&P(Y_1=y_1,\ldots,Y_n=y_n | X_1=x_1,\ldots, X_n=x_n) \\
&= h(x_1,\ldots,x_n) \cdot \prod_{i=1}^n f_i(x_i,y_i).
\end{aligned}
\end{equation}
This describes the situation in which the variable $x_i$ only affects a corresponding part of the output $y_i$, albeit in a ``relative'' sense, as before. Theorem~\ref{thm:general} directly gives the following.

\vspace{.2cm}

\begin{proposition}\label{prop:causal-independence} Write the interaction decompositions of the embeddings as
$\vu = \vu_\varnothing + \sum_{i=1}^k \vu_{i} + \tilde{\vu}$ and $\vv = \vv_\varnothing + \sum_{i=1}^k \vv_{i} + \tilde{\vv}$, where $\tilde \vu, \tilde \vv$ collect all components of order at least two.
Then the condition~\eqref{eq:causal-independence} is equivalent to:
\begin{enumerate}
    \item $\langle \vu, \tilde \vv \rangle = \langle \tilde \vu, \vv - \vv_{\varnothing} \rangle = 0$.
    \item $\langle \vu_{i}, \vv_{j} \rangle = 0$ unless $i=j$.
\end{enumerate}
\end{proposition}

Other sets of independence constraints can be considered and similarly characterized geometrically using Theorem~\ref{thm:general}.

\section{Qualitative examples and discussion}

In this section, we discuss a few informal examples of data structures reflected in model embeddings. 
\vspace{.4cm}

\begin{example}[Analogies.]\label{ex:analogies}
Consider the typical example of word analogies: $\Z = \{w, m, q, k\}$ (woman, man, queen, king). This set of words is naturally viewed as a product of variables $\Z = \Z_1 \times \Z_2$ with $\Z_1=\{\text{female}, \text{male}\}$ and $\Z_2=\{\text{non-royal},\text{royal}\}$. If $\vw: \Z \rightarrow V$ is any vector embedding, then the interaction decomposition of the vector of (say) `woman' is given by $\vw({w}) = \sum_{I} \vw_I({w})$ where
\[
\begin{aligned}
&\vw_{\varnothing}({w}) = \frac{1}{4}\left(\vw({w}) + \vw({m}) + \vw({q}) + \vw({k})\right)\\
&\vw_{1}({w}) = \frac{1}{4}\left(\vw({w}) - \vw({m}) + \vw({q}) - \vw({k})\right)\\
&\vw_{2}({w}) = \frac{1}{4}\left(\vw({w}) + \vw({m}) - \vw({q}) - \vw({k})\right)\\
&\vw_{\{1,2\}}({w}) = \frac{1}{4}\left(\vw({w}) - \vw({m}) - \vw({q}) + \vw({k})\right).\\
\end{aligned}
\]
The pairwise interaction component $\vw_{\{1,2\}}({w})$ is, up to a scalar factor, precisely the vector which is required to be zero in the classical ``parallelogram'' relation ($\vw(w) - \vw(m) = \vw(q) - \vw(k)$). The same is true for the representations of $m, q, k$. Thus, linear analogy relations arise exactly when pairwise interaction components vanish. 
\end{example}

\begin{example}[Decomposable embeddings] Embeddings with only unary interactions $\vw = \vw_0 + \sum_{i \in [k]} \vw_i$ generalize the parallelogram structure of analogies and correspond to independence conditions among factors (Figure~\ref{fig:qualitative-examples}, left). Decomposability is not only about ``directions'' but requires actual \emph{equality} among differences of embedding vectors with the same differentiating factors (\eg, sides of a parallelogram). The vector components $\vw_i$ are also generally not orthogonal, but one can show that different interaction components orthogonal if and only if the inner product is invariant to permutations of factors (\ie, $\langle \vu(x), \vu(x') \rangle = \langle \vu(g(x)), \vu(g(x')) \rangle$ where $g \in \mathfrak S_{|\mathcal X_1|} \times \mathfrak S_{|\mathcal X_2|} \times \mathfrak S_{|\mathcal X_3|}$). Decomposability is best viewed as an \emph{affine} property of embeddings, involving parallelism but not distances or angles.\footnote{Indeed, as noted in~\cite{parkLinearRepresentationHypothesis2023,tragerLinearSpacesMeanings2023}, the inner product of embeddings of the same type is not meaningful in terms of probabilities. The two embeddings should actually be seen as mappings into dual vector spaces $V$ and $V^*$.} However, since non-zero interaction terms means that different factors are not independent, the norm of the interaction factors could be used as a heuristic geometric version of mutual information (Figure~\ref{fig:qualitative-examples}, right).

\begin{figure}[htbp]
    \centering
    \begin{minipage}[b]{0.65\textwidth}
        \centering
        \includegraphics[width=\textwidth,valign=c]{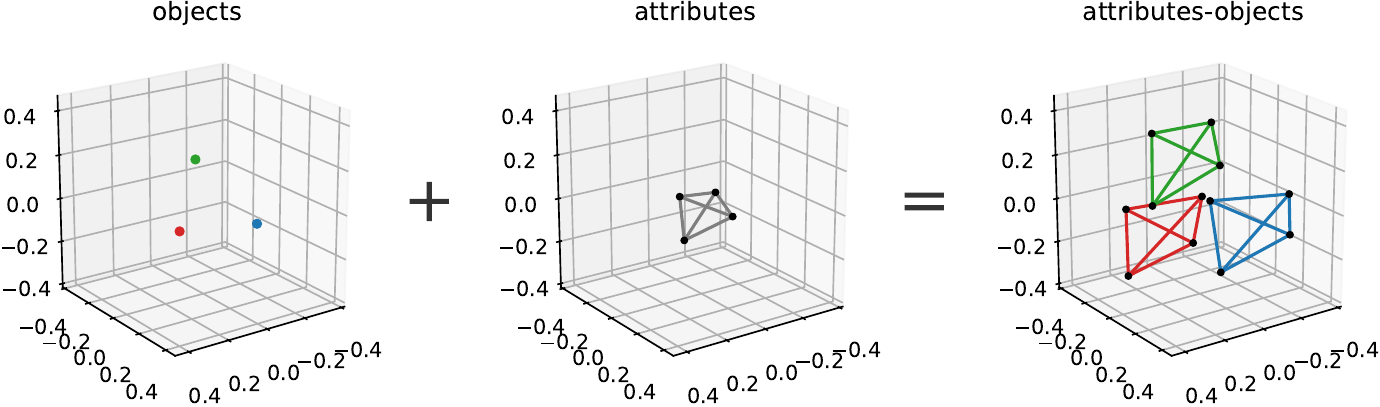}
    \end{minipage}\,\,\,\,\,
    \begin{minipage}[b]{0.29\textwidth}
        \centering
        \includegraphics[width=\textwidth,valign=c]{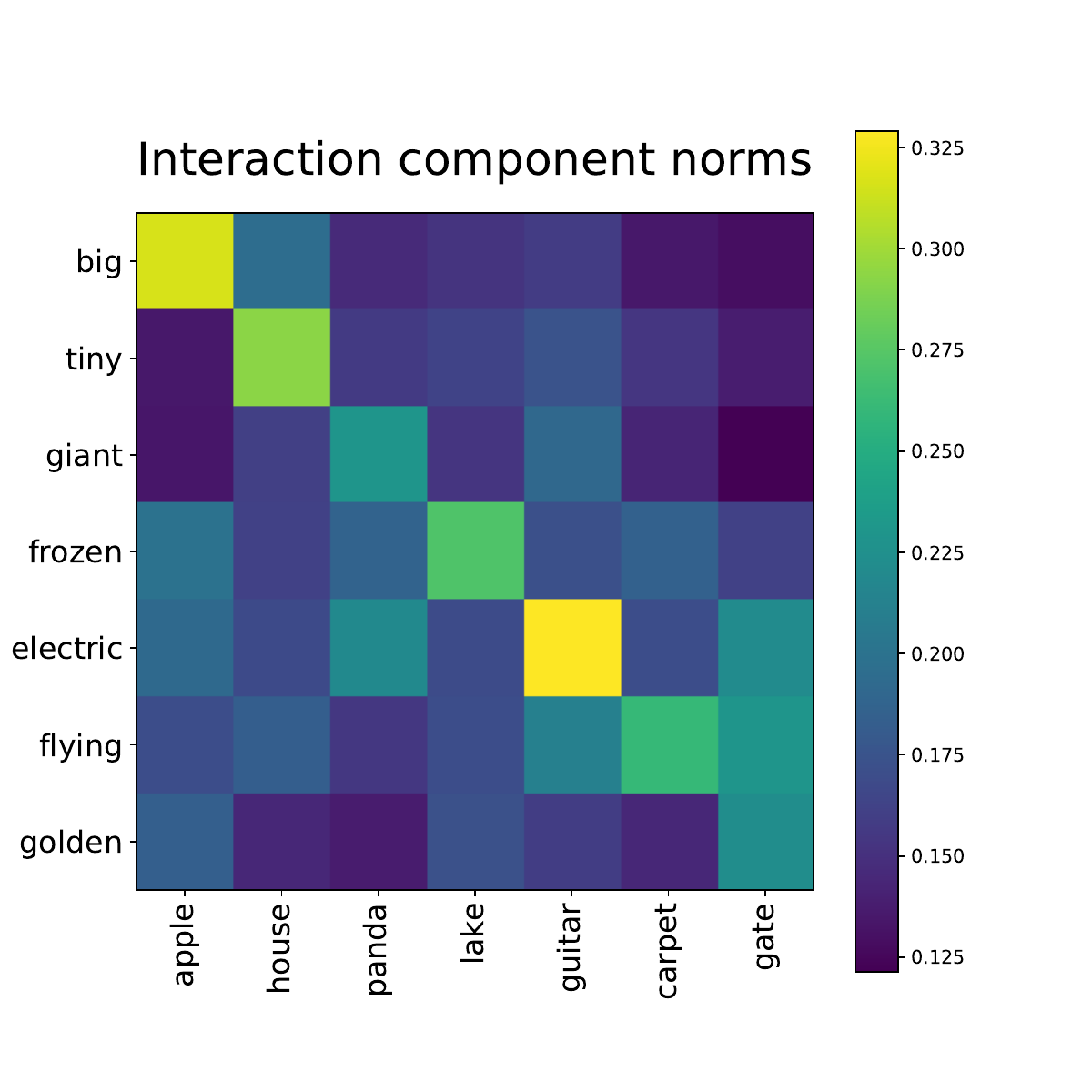}
    \end{minipage}    
    \caption{\emph{Left:} visualization of decomposable structure from object-attribute paris embedded with ST5-XL~\cite{niSentenceT5ScalableSentence2021}. 
    \emph{Right:} norm of interaction components for pairs of attributes-objects. Large norms correspond to pairs of words with strong contextual meanings.}
    \label{fig:qualitative-examples}
\end{figure}
\end{example}

\begin{example}[Conditionally independent factors.] Consider inputs $\mathcal X = \X_1 \times \X_2 \times \X_3$ corresponding to a random variable $X = (X_1,X_2,X_3)$ with $X_1 \ci X_2 \, | \, X_3$. Assume that the output embeddings are such that $Span(\vv(y) - \vv(y') \colon y,y' \in \mathcal Y) = V$. If this is not the case, we replace the embeddings $\vu$ with their projections onto that space. Proposition~\ref{prop:causal-independence} implies that output embeddings can be written as
\[
\vu(x_1,x_2,x_3) = \vu_0 + \vu_1(x_1) + \vu_2(x_2) + \vu_3(x_3) + \vu_{12}(x_1,x_2) + \vu_{13}(x_1,x_3).
\]
If we fix the value $x_3$, then the 
the embeddings are decomposable;
however if we vary $x_2$ they are not (\eg, if $|\X_1| = |\X_2| = 2$ then different parallelograms do not have parallel edges) unless $X_1 \ci X_2 \ci X_3$. See Figure~\ref{fig:geometric-examples} in the Appendix. This means that our framework also models ``nonlinear structures'' in embeddings as a varying conditioning contexts. 
\end{example}

\begin{example}[Grammars] Assume that $\mathcal Y$ is a set of text strings and $\mathcal X$ is a set of conditioning inputs, which may be text strings, images, or any other data type. Assume that $P_x = P(\, \cdot\, |X=x)$ is described by a {probabilistic context-free grammar} (PCFG)~\cite{Jurafsky2009}
for all $x$. 
Assume moreover that the underlying context-free grammar for this distribution does not depend on $x$; this means that only conditional probabilities of expansions depend on $x$. Then for every derivable string of (terminal and non-terminal) symbols $t = A_1 \ldots A_n$, the expansion probabilities of the string depend on the probabilities of the symbols $A_i$ independently. This allows us to apply Proposition~\ref{prop:conditional-independence}, viewing $\mathcal Y_i$ as the set of possible expansions of $A_i$. Thus, if a string $s$ is an expansion of $t$ in a unique way, then we can decompose its embedding as
\[
\vv(s) = \vv_{t,\varnothing} + \vv_{t, A_1}(s) + \ldots + \vv_{t,A_n}(s).
\]
If the parsing of $s$ is ambiguous from $t$ then its representation will be a ``superposition'' of representations of this type.
 We can also iterate this approach and obtain a decomposition of the form
\[
\vv(s) = \vv_0  + \sum_{i} \vv_{\alpha_i \rightarrow \beta_i}(s)
\]
where $\alpha_i \rightarrow \beta_i$ is a sequence of derivations. This discussion suggests that syntactic structure will be at least partially reflected in language embeddings in a linear fashion.
\end{example}

\begin{example}[Vision-Language models.] In vision-language model such as CLIP~\cite{radford2021learning}, a set of images $\mathcal X$ and text $\mathcal Y$ are represented as vectors using embeddings $\vu: \mathcal X \rightarrow V$ and $\vv: \mathcal Y \rightarrow V$ that capture conditional probabilities $p(y|x)$ and $p(x|y)$ as in~\eqref{eq:distribution}. It has been empirically observed that embeddings obtained in this way are not ``aligned,'' in the sense that it need not be true that an image and its corresponding best caption are close in embedding space~\cite{liangMindGapUnderstanding2022}. On the other hand, other recent works have observed that embeddings of images and text can be interchanged in some cases~\cite{nukraiTextOnlyTrainingImage2022, tragerLinearSpacesMeanings2023}, suggesting at least some amount of structural alignment. We believe that our framework can be used to gain insights on these phenomena. Specifically, consider a set of paired images and captions jointly factored $x = (x_1,\ldots,x_k)$ and $y = (y_1,\ldots,y_k)$ for example object, style, background, color, etc. In this setting, it is reasonable to assume a ``disentanglement'' probabilistic condition of the form
\[
p(y|x) \propto \prod_i f_i(y_i,x_i).
\]
This allows us to apply Proposition~\ref{prop:causal-independence} which describes structural conditions between the embeddings from the two modalities. In particular, the first-order components for the two modalities are paired from the second condition of Proposition~\ref{prop:causal-independence}, thus providing a weak form of alignment. For example, if factor variables are binary attributes, then first order interaction terms form a dual vector basis.
\end{example}

\begin{example}[Compositional structures during training.] \label{ex:dynamics} We present a small experiment to validate our results in a synthetic setting. We generate data by sampling from a categorical distribution $P(z_1,z_2,z_3)$ over $\Z^3$, with $\Z = \{1,\ldots,10\}$ and train a small transformer to predict $z_3$ given $z_1, z_2$ (see the Appendix for experimental details). We construct $P$ so that $z_1, z_2$ are conditionally independent given $z_3$. Proposition~\ref{prop:causal-independence} then says says that the projections of input embeddings $\vu(z_1,z_2)$ onto the space $Span\{\vv(z) - \vv(z') \colon z \in \mathcal Z\}$ should be decomposable. Our experiment shown in Figure~\ref{fig:train-interactions} confirms this, as can be seen from the relative norm of the pairwise interaction terms throughout training (the curve for $[1,1]$) and the geometry of a projected set of paired $2 \times 2$ inputs. However, the plots on the left also shows that embeddings are roughly decomposable even at beginning of the training process. Indeed, it was already observed in~\cite{tragerLinearSpacesMeanings2023} that a \emph{randomly initialized} transformer encoder exhibits decomposable structures when factors are aligned with tokens (\ie, when each factor corresponds to a substring). This is a useful bias of transformers, as it encourages words to be processed compositionally. In the plots shown at the center, we use a categorical distribution where the $z_1, z_2$ are conditionally independent but not in a token-aligned manner (we apply a permutation to the set of all pairs $(z_1,z_2)$). In this case, the decomposable structure is not present at initialization but ``emerges'' with training. Finally, in the plot on the right, we use a distribution without the factored structure but consider aligned interactions when $z_1,z_2$ take the same values. Here we see that the embeddings are roughly decomposable at initialization however this structure is destroyed during training.
\end{example}

\begin{figure}[htbp]
    \centering
    \includegraphics[width=0.25\textwidth]{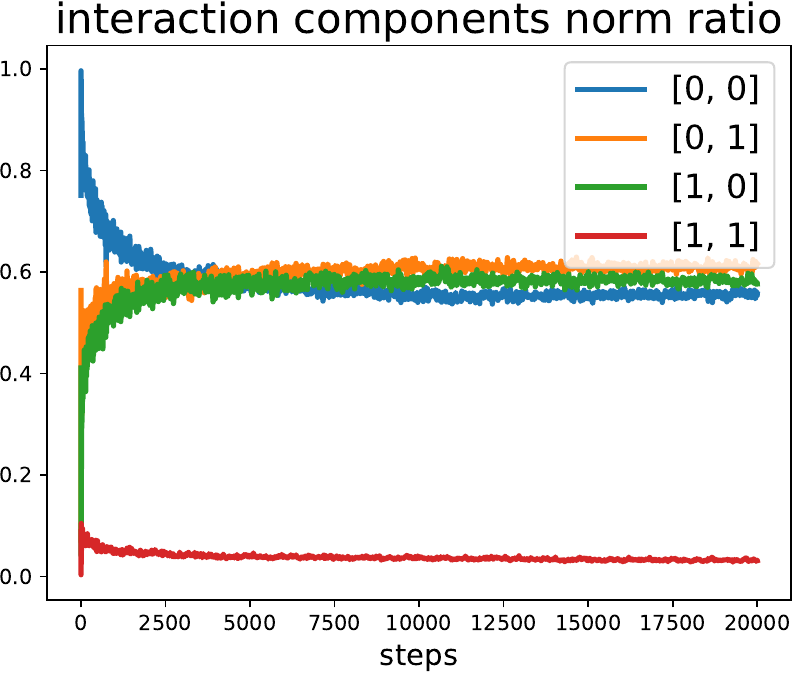}\qquad
    \includegraphics[width=0.25\textwidth]{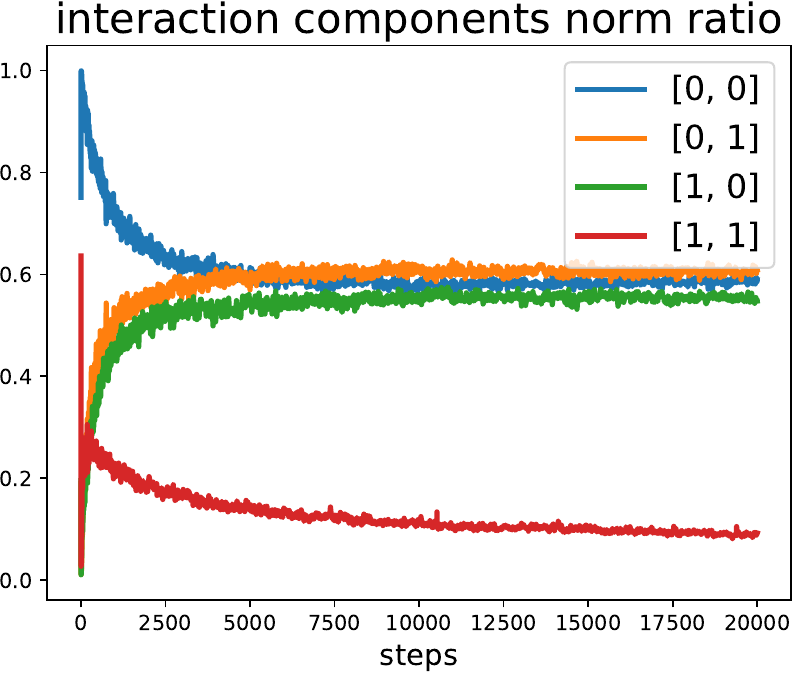}\qquad
    \includegraphics[width=0.25\textwidth]{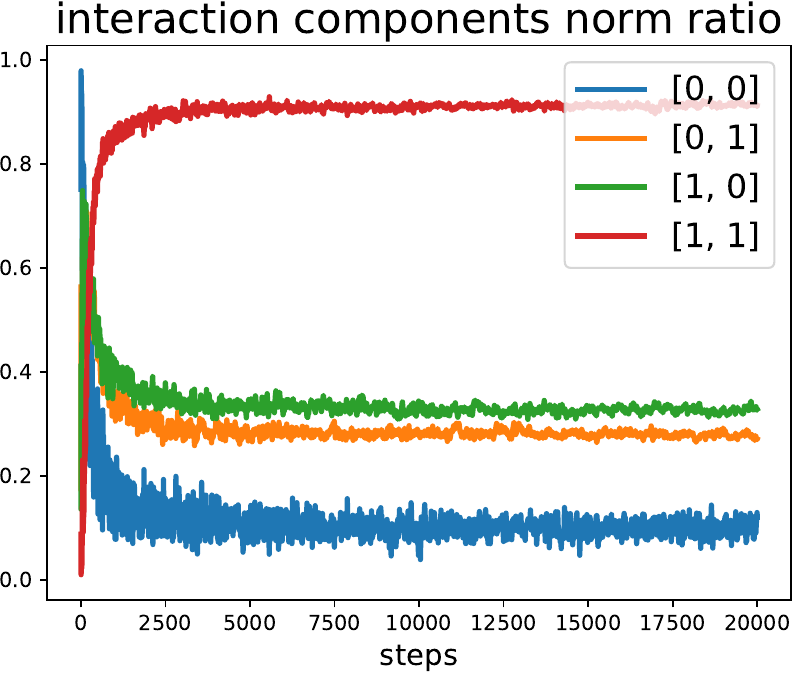}\\[.2cm]
    \includegraphics[width=0.25\textwidth]{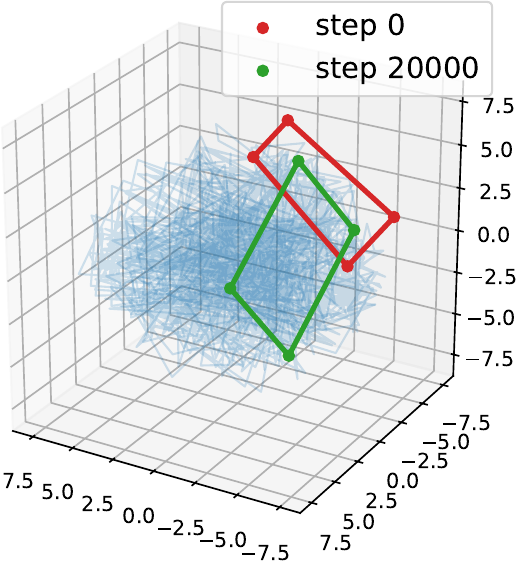}\qquad
    \includegraphics[width=0.25\textwidth]{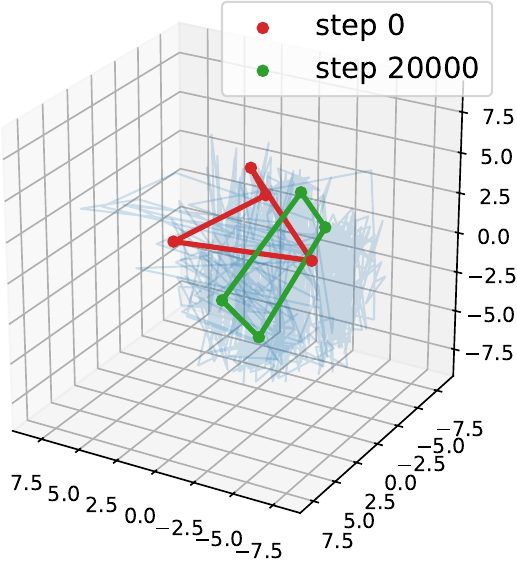}\qquad
    \includegraphics[width=0.25\textwidth]{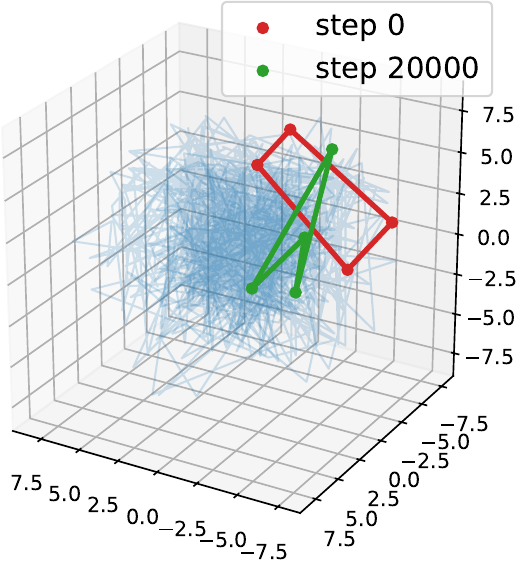}
    \caption{Compositional structures throughout training. The top row shows the evolution of the norm of interaction components for the two input tokens ([1,1] corresponding to the pairwise interaction $\vu_{12}$) and the second row shows projections of embeddings. \emph{Left (top and bottom):} for factors aligned with tokenization (``syntactic factors''), the decomposable structure is present at initialization. \emph{Center (top and bottom):} for factors that do not correspond to tokens (``semantic factors'') the decomposable structure is emergent. \emph{Right (top and bottom):} if the probability is not factored, the decomposable structure is destroyed by the training process.}
    \label{fig:train-interactions}
\end{figure}

\section{Conclusions}
\label{sec:conclusions}

We have given a description of linear compositional structures in neural
embeddings in terms of interaction decompositions and showed a precise
correspondence between these structures and probabilistic constraints on data
distributions.

One limitation of our presentation is that products of finite sets may not
always be suited for modeling complex relational structures. Our setup also
considers only the embeddings prior to the final softmax, so our results do
not directly describe structures in intermediate layers. Furthermore, our
description does not account for the learning and the training process, which
are likely important for regularization, as seen in Example~\ref{ex:dynamics}
and also argued in~\cite{jiangOriginsLinearRepresentations2024a}.

Our framework can be compared with the compositional vector representations of
text considered in~\cite{coeckeMathematicalFoundationsCompositional2010a}. In
that approach, a representation is constructed using tensor products from
representations associated with constituent parts of the text. While this is
theoretically well-motivated, it requires a grammar to be fixed beforehand. In
contrast, we seek to ``deconstruct'' a given single representation (or a pair
of representations) that models a distribution with latent structure. The two
perspectives are however connected, as tensor decompositions implicitly appear
in our framework as well --- for example, exponentiating decomposable
representations yields tensors with rank-one slices. Interaction
decompositions are in fact related to the geometry of exponential families
(see Appendix~\ref{sec:exponential}). Ideas from information
geometry~\cite{ayInformationGeometry2017a} and algebraic
statistics~\cite{pachterAlgebraicStatisticsComputational} could be used to
further study the discrete combinatorial structures in neural embeddings
described in this work.

\bibliographystyle{plain}
\bibliography{references}

\begin{thebibliography}{10}

\bibitem{allenAnalogiesExplainedUnderstanding}
Carl Allen and Timothy Hospedales.
\newblock Analogies {Explained}: {Towards} {Understanding} {Word} {Embeddings}.
\newblock page~9, 2019.

\bibitem{aroraLatentVariableModel2016}
Sanjeev Arora, Yuanzhi Li, Yingyu Liang, Tengyu Ma, and Andrej Risteski.
\newblock A latent variable model approach to pmi-based word embeddings.
\newblock {\em Transactions of the Association for Computational Linguistics},
  4:385--399, 2016.

\bibitem{ayInformationGeometry2017a}
Nihat Ay, Jürgen Jost, Hông~Vân Lê, and Lorenz Schwachhöfer.
\newblock {\em Information {Geometry}}, volume~64 of {\em Ergebnisse der
  {Mathematik} und ihrer {Grenzgebiete} 34}.
\newblock Springer International Publishing, Cham, 2017.

\bibitem{baroni2010nouns}
Marco Baroni and Roberto Zamparelli.
\newblock Nouns are vectors, adjectives are matrices: Representing
  adjective-noun constructions in semantic space.
\newblock In {\em Proceedings of the 2010 conference on empirical methods in
  natural language processing}, pages 1183--1193, 2010.

\bibitem{benderClimbingNLUMeaning2020}
Emily~M. Bender and Alexander Koller.
\newblock Climbing towards {NLU}: {On} {Meaning}, {Form}, and {Understanding}
  in the {Age} of {Data}.
\newblock In {\em Proceedings of the 58th {Annual} {Meeting} of the
  {Association} for {Computational} {Linguistics}}, pages 5185--5198, Online,
  2020. Association for Computational Linguistics.

\bibitem{bricken2022monosemanticity}
Trenton Bricken, Adly Templeton, Joshua Batson, Brian Chen, Adam Jermyn, Tom
  Conerly, Nick Turner, Cem Anil, Carson Denison, Amanda Askell, Robert
  Lasenby, Yifan Wu, Shauna Kravec, Nicholas Schiefer, Tim Maxwell, Nicholas
  Joseph, Zac Hatfield-Dodds, Alex Tamkin, Karina Nguyen, Brayden McLean,
  Josiah~E Burke, Tristan Hume, Shan Carter, Tom Henighan, and Christopher
  Olah.
\newblock Towards monosemanticity: Decomposing language models with dictionary
  learning.
\newblock {\em Transformer Circuits Thread}, 2023.

\bibitem{clarkVectorSpaceModels2015a}
Stephen Clark.
\newblock Vector {Space} {Models} of {Lexical} {Meaning}.
\newblock In Shalom Lappin and Chris Fox, editors, {\em The {Handbook} of
  {Contemporary} {Semantic} {Theory}}, pages 493--522. Wiley, 1 edition,
  September 2015.

\bibitem{coeckeMathematicalFoundationsCompositional2010a}
Bob Coecke, Mehrnoosh Sadrzadeh, and Stephen Clark.
\newblock Mathematical {Foundations} for a {Compositional} {Distributional}
  {Model} of {Meaning}, March 2010.
\newblock arXiv:1003.4394 [cs, math].

\bibitem{darrochAdditiveMultiplicativeModels1983}
J.~N. Darroch and T.~P. Speed.
\newblock Additive and {Multiplicative} {Models} and {Interactions}.
\newblock {\em The Annals of Statistics}, 11(3), September 1983.

\bibitem{devlinBERTPretrainingDeep2019}
Jacob Devlin, Ming-Wei Chang, Kenton Lee, and Kristina Toutanova.
\newblock {BERT}: {Pre}-training of {Deep} {Bidirectional} {Transformers} for
  {Language} {Understanding}.
\newblock {\em arXiv:1810.04805 [cs]}, May 2019.
\newblock arXiv: 1810.04805.

\bibitem{dosovitskiyImageWorth16x162021}
Alexey Dosovitskiy, Lucas Beyer, Alexander Kolesnikov, Dirk Weissenborn,
  Xiaohua Zhai, Thomas Unterthiner, Mostafa Dehghani, Matthias Minderer, Georg
  Heigold, Sylvain Gelly, Jakob Uszkoreit, and Neil Houlsby.
\newblock An {Image} is {Worth} 16x16 {Words}: {Transformers} for {Image}
  {Recognition} at {Scale}.
\newblock {\em arXiv:2010.11929 [cs]}, June 2021.
\newblock arXiv: 2010.11929.

\bibitem{elhage2022superposition}
Nelson Elhage, Tristan Hume, Catherine Olsson, Nicholas Schiefer, Tom Henighan,
  Shauna Kravec, Zac Hatfield-Dodds, Robert Lasenby, Dawn Drain, Carol Chen,
  Roger Grosse, Sam McCandlish, Jared Kaplan, Dario Amodei, Martin Wattenberg,
  and Christopher Olah.
\newblock Toy models of superposition.
\newblock {\em Transformer Circuits Thread}, 2022.

\bibitem{ethayarajhUnderstandingLinearWord2019}
Kawin Ethayarajh, David Duvenaud, and Graeme Hirst.
\newblock Towards {Understanding} {Linear} {Word} {Analogies}, August 2019.
\newblock arXiv:1810.04882 [cs].

\bibitem{ferroneSymbolicDistributedDistributional2020}
Lorenzo Ferrone and Fabio~Massimo Zanzotto.
\newblock Symbolic, {Distributed}, and {Distributional} {Representations} for
  {Natural} {Language} {Processing} in the {Era} of {Deep} {Learning}: {A}
  {Survey}.
\newblock {\em Frontiers in Robotics and AI}, 6:153, January 2020.

\bibitem{gittensSkipGramZipfUniform2017}
Alex Gittens, Dimitris Achlioptas, and Michael~W. Mahoney.
\newblock Skip-{Gram} - {Zipf} + {Uniform} = {Vector} {Additivity}.
\newblock In {\em Proceedings of the 55th {Annual} {Meeting} of the
  {Association} for {Computational} {Linguistics} ({Volume} 1: {Long}
  {Papers})}, pages 69--76, Vancouver, Canada, 2017. Association for
  Computational Linguistics.

\bibitem{gurneeLanguageModelsRepresent2023}
Wes Gurnee and Max Tegmark.
\newblock Language {Models} {Represent} {Space} and {Time}, October 2023.
\newblock arXiv:2310.02207 [cs].

\bibitem{higgins2018towards}
Irina Higgins, David Amos, David Pfau, Sebastien Racaniere, Loic Matthey,
  Danilo Rezende, and Alexander Lerchner.
\newblock Towards a definition of disentangled representations.
\newblock {\em arXiv preprint arXiv:1812.02230}, 2018.

\bibitem{hinton1984distributed}
Geoffrey~E Hinton.
\newblock Distributed representations.
\newblock 1984.

\bibitem{jiangOriginsLinearRepresentations2024a}
Yibo Jiang, Goutham Rajendran, Pradeep Ravikumar, Bryon Aragam, and Victor
  Veitch.
\newblock On the {Origins} of {Linear} {Representations} in {Large} {Language}
  {Models}, March 2024.
\newblock arXiv:2403.03867 [cs, stat].

\bibitem{Jurafsky2009}
Dan Jurafsky and James~H. Martin.
\newblock {\em Speech and language processing : an introduction to natural
  language processing, computational linguistics, and speech recognition}.
\newblock Pearson Prentice Hall, Upper Saddle River, N.J., 2009.

\bibitem{lauritzen1996graphical}
Steffen~L Lauritzen.
\newblock {\em Graphical models}, volume~17.
\newblock Clarendon Press, 1996.

\bibitem{li2022emergent}
Kenneth Li, Aspen~K Hopkins, David Bau, Fernanda Vi{\'e}gas, Hanspeter Pfister,
  and Martin Wattenberg.
\newblock Emergent world representations: Exploring a sequence model trained on
  a synthetic task.
\newblock {\em arXiv preprint arXiv:2210.13382}, 2022.

\bibitem{liangMindGapUnderstanding2022}
Weixin Liang, Yuhui Zhang, Yongchan Kwon, Serena Yeung, and James Zou.
\newblock Mind the {Gap}: {Understanding} the {Modality} {Gap} in {Multi}-modal
  {Contrastive} {Representation} {Learning}, October 2022.
\newblock arXiv:2203.02053 [cs].

\bibitem{linardatosExplainableAIReview2020}
Pantelis Linardatos, Vasilis Papastefanopoulos, and Sotiris Kotsiantis.
\newblock Explainable {AI}: {A} {Review} of {Machine} {Learning}
  {Interpretability} {Methods}.
\newblock {\em Entropy}, 23(1):18, December 2020.

\bibitem{lu2022unified}
Jiasen Lu, Christopher Clark, Rowan Zellers, Roozbeh Mottaghi, and Aniruddha
  Kembhavi.
\newblock Unified-io: A unified model for vision, language, and multi-modal
  tasks.
\newblock {\em arXiv preprint arXiv:2206.08916}, 2022.

\bibitem{mikolovDistributedRepresentationsWords2013}
Tomas Mikolov, Ilya Sutskever, Kai Chen, Greg~S. Corrado, and Jeff Dean.
\newblock Distributed representations of words and phrases and their
  compositionality.
\newblock In {\em Advances in neural information processing systems}, pages
  3111--3119, 2013.

\bibitem{mitchell2008vector}
Jeff Mitchell and Mirella Lapata.
\newblock Vector-based models of semantic composition.
\newblock In {\em proceedings of ACL-08: HLT}, pages 236--244, 2008.

\bibitem{nandaEmergentLinearRepresentations2023}
Neel Nanda, Andrew Lee, and Martin Wattenberg.
\newblock Emergent {Linear} {Representations} in {World} {Models} of
  {Self}-{Supervised} {Sequence} {Models}, September 2023.
\newblock arXiv:2309.00941 [cs].

\bibitem{niSentenceT5ScalableSentence2021}
Jianmo Ni, Gustavo~Hernández Ábrego, Noah Constant, Ji~Ma, Keith~B. Hall,
  Daniel Cer, and Yinfei Yang.
\newblock Sentence-{T5}: {Scalable} {Sentence} {Encoders} from {Pre}-trained
  {Text}-to-{Text} {Models}, December 2021.
\newblock arXiv:2108.08877 [cs].

\bibitem{nukraiTextOnlyTrainingImage2022}
David Nukrai, Ron Mokady, and Amir Globerson.
\newblock Text-{Only} {Training} for {Image} {Captioning} using
  {Noise}-{Injected} {CLIP}.
\newblock November 2022.
\newblock arXiv:2211.00575 [cs].

\bibitem{pachterAlgebraicStatisticsComputational}
Lior Pachter and Bernd Sturmfels.
\newblock Algebraic {Statistics} for {Computational} {Biology}.
\newblock page 447.

\bibitem{parkLinearRepresentationHypothesis2023}
Kiho Park, Yo~Joong Choe, and Victor Veitch.
\newblock The {Linear} {Representation} {Hypothesis} and the {Geometry} of
  {Large} {Language} {Models}, November 2023.
\newblock arXiv:2311.03658 [cs, stat].

\bibitem{radford2021learning}
Alec Radford, Jong~Wook Kim, Chris Hallacy, Aditya Ramesh, Gabriel Goh,
  Sandhini Agarwal, Girish Sastry, Amanda Askell, Pamela Mishkin, Jack Clark,
  et~al.
\newblock Learning transferable visual models from natural language
  supervision.
\newblock In {\em International conference on machine learning}, pages
  8748--8763. PMLR, 2021.

\bibitem{radfordUnsupervisedRepresentationLearning2016}
Alec Radford, Luke Metz, and Soumith Chintala.
\newblock Unsupervised {Representation} {Learning} with {Deep} {Convolutional}
  {Generative} {Adversarial} {Networks}, January 2016.
\newblock arXiv:1511.06434 [cs].

\bibitem{raukerTransparentAISurvey2023}
Tilman Räuker, Anson Ho, Stephen Casper, and Dylan Hadfield-Menell.
\newblock Toward {Transparent} {AI}: {A} {Survey} on {Interpreting} the {Inner}
  {Structures} of {Deep} {Neural} {Networks}, August 2023.
\newblock arXiv:2207.13243 [cs].

\bibitem{smolenskyTensorProductVariable1990}
Paul Smolensky.
\newblock Tensor product variable binding and the representation of symbolic
  structures in connectionist systems.
\newblock {\em Artificial Intelligence}, 46(1-2):159--216, November 1990.

\bibitem{tragerLinearSpacesMeanings2023}
Matthew Trager, Pramuditha Perera, Luca Zancato, Alessandro Achille, Parminder
  Bhatia, and Stefano Soatto.
\newblock Linear {Spaces} of {Meanings}: {Compositional} {Structures} in
  {Vision}-{Language} {Models}, March 2023.
\newblock arXiv:2302.14383 [cs].

\bibitem{turneyFrequencyMeaningVector2010}
Peter~D. Turney and Patrick Pantel.
\newblock From {Frequency} to {Meaning}: {Vector} {Space} {Models} of
  {Semantics}.
\newblock {\em Journal of Artificial Intelligence Research}, 37:141--188,
  February 2010.
\newblock arXiv:1003.1141 [cs].

\bibitem{vaswaniAttentionAllYou2017}
Ashish Vaswani, Noam Shazeer, Niki Parmar, Jakob Uszkoreit, Llion Jones,
  Aidan~N. Gomez, Lukasz Kaiser, and Illia Polosukhin.
\newblock Attention {Is} {All} {You} {Need}.
\newblock {\em arXiv:1706.03762 [cs]}, December 2017.
\newblock arXiv: 1706.03762.

\bibitem{verma2021audio}
Prateek Verma and Jonathan Berger.
\newblock Audio transformers: Transformer architectures for large scale audio
  understanding. adieu convolutions.
\newblock {\em arXiv preprint arXiv:2105.00335}, 2021.

\bibitem{wangConceptAlgebraScoreBased2023}
Zihao Wang, Lin Gui, Jeffrey Negrea, and Victor Veitch.
\newblock Concept {Algebra} for {Score}-{Based} {Conditional} {Models}, July
  2023.
\newblock arXiv:2302.03693 [cs, stat].

\bibitem{zouRepresentationEngineeringTopDown2023a}
Andy Zou, Long Phan, Sarah Chen, James Campbell, Phillip Guo, Richard Ren,
  Alexander Pan, Xuwang Yin, Mantas Mazeika, Ann-Kathrin Dombrowski, Shashwat
  Goel, Nathaniel Li, Michael~J. Byun, Zifan Wang, Alex Mallen, Steven Basart,
  Sanmi Koyejo, Dawn Song, Matt Fredrikson, J.~Zico Kolter, and Dan Hendrycks.
\newblock Representation {Engineering}: {A} {Top}-{Down} {Approach} to {AI}
  {Transparency}, October 2023.
\newblock arXiv:2310.01405 [cs].

\end{thebibliography}
\newpage
\appendix

\section{Proofs}
\label{sec:proofs}

\begin{proof}[Proof of Proposition~\ref{prop:projections}] The result is essentially the same as Proposition~2.17 in~\cite{ayInformationGeometry2017a}, but we sketch an alternative (and arguably simpler) proof for completeness.

We write $V^{\Z} \cong \RR^{n_1} \otimes \ldots \otimes \RR^{n_k} \otimes V$ where $n_i = |\Z_i|$. We also consider decompositions of the form $\RR^{n_i} \cong W_{n_i,0} \oplus W_{n_i,1}$ where $W_{n_i,0} = Span\{(1,\ldots,1)^\top\}$ and $W_{n_i,1} = \{ v \colon \sum_{j=1}^{n_i} v_j = 0\}$ (``trivial'' and ``standard'' representations of $\mathfrak S_{n_i}$). This yields
\[
\begin{aligned}
V^{\Z} &\cong (W_{n_1,0} \oplus W_{n_1,1}) \otimes \cdots (W_{n_k,0} \oplus W_{n_k,1}) \otimes V\\[.2cm]
&\cong \bigoplus_{\epsilon \in \{0,1\}^k} W_{n_1,\epsilon_1} \otimes \cdots \otimes W_{n_k,\epsilon_k} \otimes V.
\end{aligned}
\]
For each $I \subset [k]$, the pure interaction space $E_I$ corresponds to the summand $W_{n_1,\epsilon_1} \otimes \cdots \otimes W_{n_k,\epsilon_k} \otimes V$ where $\epsilon_i = 1$ if $i \in I$ and $\epsilon_i = 0$ otherwise. The projection onto such space can be written as $\tau_{n_1,\epsilon_1} \otimes \ldots \otimes \tau_{n_k,\epsilon_k} \otimes {\rm Id}_V$ where each $\tau_{n_i, \epsilon_i}: \RR^{n_i} \rightarrow \RR^{n_i}$ is the projection onto $W_{n_i,\epsilon_i}$, described by
\[
\tau_{n_i,0}(v) = \left(\frac{1}{n_i} \sum_{j=1}^{n_i} v_j\right) \cdot (1,\ldots,1)^\top \qquad \text{or} \qquad \tau_{n_i,1}(v) = v - \tau_{n_i,0}(v).
\]
We now observe that 
\[
\begin{aligned}
\tau_{n_1,\epsilon_1} \otimes \ldots \otimes \tau_{n_k,\epsilon_k} \otimes {\rm Id}_V = \sum_{\substack{\delta \in \{0,1\}^k \\[.1cm] \delta \le \epsilon}} (-1)^{|\epsilon| - |\delta|} \tilde \tau_{n_1,\delta_1} \otimes \ldots \otimes \tilde \tau_{n_k,\delta_k} \otimes {\rm Id_V},
\end{aligned}
\]
\[
\text{where} \qquad \tilde \tau_{n_i,0} = \tau_{n_i,0} \text{ and } \tilde \tau_{n_i,1} = {\rm Id}_{n_i}.
\]
and that $\tilde \tau_{n_1,\delta_1} \otimes \ldots \otimes \tilde \tau_{n_k,\delta_k} \otimes {\rm Id_V}$ corresponds to the map $\pi_J(\bm w) = \frac{|\Z_J|}{|\Z|}\sum_{z_{[k] \setminus J} \in \Z_{[k] \setminus J}} \vw(z_J, z_{[k] \setminus J})$ with $J = \{i \in [k] \colon \delta_i = 1\}$. All claims of Proposition~\ref{prop:projections} now follow.
\end{proof}

\begin{proof}[Proof of Lemma~\ref{lemma:cond-ind}] If~\eqref{eq:factorization-exists} holds, then it is enough to set $p(x) \propto h(x_1,\ldots,x_m)^{-1}$. Conversely, $Z_A \ci Z_B \, |_X Z_C$ means that there exists $p$ such that $P(Y|X)p(x) = f(z_A, z_C) g(z_B, z_C)$, which implies that $P(Y|X) = f(z_A, z_C) g(z_B, z_C) p(x)^{-1}$.
\end{proof}

To prove our main result we use the following basic fact.

\begin{lemma}\label{lemma:interaction-basic} Let $\mathcal S$ be a family of subsets of $[k]$. A map $\vw: \Z_1 \times \ldots \times \Z_k \rightarrow V$ can be written as
$\vw = \sum_{I \in \mathcal S} f_I(z_I)$ for some functions $f_I: \Z_I \rightarrow V$ if and only if the interaction components of $\vw$ are such that $\vw_J=0$ whenever $J$ is not contained in any set in $\mathcal S$.
\end{lemma}
\begin{proof} It is clear that the condition is sufficient. To show that it is necessary, we note that $\pi_I(f_I) = f_I$, so it is enough show that $Q_J \pi_I = 0$ whenever $J \not \subset I$. It follows from the M\"obius inversion formula (\cite[Lemma 2.12]{ayInformationGeometry2017a}) that $\pi_I = \sum_{J' \subset I} Q_{J'}$. Since $Q_J Q_{J'} = 0$ whenever $J \ne J'$, the claim follows.
\end{proof}

\begin{proof}[Proof of Theorem~\ref{thm:general}] Assume first that the interaction decompositions $\vu = \sum_{I \subset [m]} \vu_I$ and $\vv = \sum_{J \subset [n]} \vv_J$ satisfy
\begin{equation}\label{eq:interaction-vanishing-p}
\langle \vu_I, \vv_J\rangle = 0, 
\end{equation}
for all $I \subset [m]$, $J \subset [n]$ with $J \ne \varnothing$ such that $(I \sqcup J) \cap A \ne \varnothing$ and $(I \sqcup J) \cap B \ne \varnothing$. Then we have that
\[
\begin{aligned}
\log P(Y|X) =& \sum_{I\sqcup J \subset A \cup C} \langle \vu_I(x), \vv_J(y) \rangle \\
&+ \sum_{I\sqcup J \subset B \cup C} \langle \vu_I(x), \vv_J(y) \rangle \\
&- \sum_{I\sqcup J \subset C} \langle \vu_I(x), \vv_J(y) \rangle \\
&+ \langle \vu(x), \vv_{\varnothing} \rangle - \psi(x),
\end{aligned}
\]
where $\psi(x) := \log \sum_{y' \in \mathcal Y} \exp \langle \vu({x}), \vv({y'})\rangle$. Exponentiating, we see that $P(Y|X)$ this has the same form as~\eqref{eq:factorization-exists}. Conversely, assume that we can write
\begin{equation}\label{eq:decomposition-product}
\begin{aligned}
&\sum_{I \subset[m], J \subset[n]} \langle \vu_I(x), \vv_J(y) \rangle  =\\
&\quad=\tilde f(z_A, z_C) + \tilde g(z_B, z_C) + \tilde h(x_1,\ldots,x_m).
\end{aligned}
\end{equation}
If we treat $\vw: (x,y) \mapsto \langle \vu(x), \vv(y) \rangle$ as a map $\X_1 \times \ldots \times \X_m  \times \Y_1\times \ldots \times \Y_n \rightarrow \RR$, then it is easy to see that $\langle \vu_I, \vv_J\rangle$ is the $(I \sqcup J)$-interaction component of $\vw$. Indeed, writing $Q_I^\X$, $Q_J^\Y$ for the projection operators on the $I, J$ components for $\X$, $\Y$ respectively (cf. Proposition~\ref{prop:projections}), we have that
\[
\begin{aligned}
\langle \vu_I, \vv_J \rangle  &= \langle Q_I^\X \vu, Q_J^\Y \vv \rangle\\[.2cm]
&=\langle \sum_{I' \subset I} (-1)^{|I \setminus I'|} \pi_{I'}^\X \vu, \sum_{J' \subset J} (-1)^{|J \setminus J'|} \pi_{J'}^\Y \vv\rangle \\[.2cm]
&=\sum_{I' \subset I} \sum_{J' \subset J} (-1)^{|I \setminus I'| + |J \setminus J'|} \langle \pi_{I'}^\X \vu,  \pi_{J'}^\Y \vv\rangle\\[.2cm]
&=\sum_{I' \subset I} \sum_{J' \subset J} (-1)^{|I \sqcup J \setminus I' \sqcup J'|} \pi_{I' \sqcup J'}^{\X\times \Y} \langle \vu(x),  \vv(y)\rangle \\[.2cm]
&=Q_{(I,J)}^{\X \times \Y} \langle \vu, \vv \rangle,
\end{aligned}
\]
where $Q_{I \sqcup J}^{\X \times \Y}$ is the projection onto the $I \sqcup J$-component for $\X \times \Y$. It follows now from Lemma~\ref{lemma:interaction-basic} above that $\langle \vu_I, \vw_J \rangle \ne 0$ implies $I \sqcup J \subset A \cup C$ or $I \sqcup J \subset B \cup C$ or $J = \varnothing$, as desired.
\end{proof}

\begin{proof}[Proof of Proposition~\ref{prop:conditional-independence}.] The claim follows from Theorem~\ref{thm:general} applied to $P_x$ (which has the form~\eqref{eq:distribution} for $\mathcal X= \{x\}$) with $m=1$. 
\end{proof}

\begin{proof}[Proof of Proposition~\ref{prop:relative-causal-independence}] It follows from Theorem~\ref{thm:general} applied to the restriction of $P(Y|X)$ to $Y \in \mathcal Y_0$ that relative causal independence is equivalent to
\[
\langle \vu_{H}, \vv_{|\Y_0, 1} \rangle = 0
\]
for all $H \subset [m]$ such that $H \cap I \ne \varnothing, H \cap J \ne \varnothing$
where $\vv_{|\Y_0} = \vv_{|\Y_0, \varnothing} + \vv_{|\Y_0, 1}$ is the interaction decomposition for $\vv_{|\Y_0}: \Y_0 \rightarrow V$ (the restriction of $\vv$ to $\Y_0$, with only one factor). The claim now follows by observing that since $\vv_{|\Y_0, 1}(y) = \vv(y) - \frac{1}{|\Y_0|}\sum_{y' \in \Y_0} \vv(y')$ we have $Span(\vv_{|\Y_0, 1}(y) \colon y \in \Y_0) = Span(\vv(y) - \vv(y') \colon y, y' \in \Y_0)$.
\end{proof}

\begin{proof}[Proof of Proposition~\ref{prop:causal-independence}.] The condition~\eqref{eq:causal-independence} means that $Z_A \ci Z_B \mid_X Z_C$ unless $Z_A \sqcup Z_B \subset \{X_i, Y_i\}$. By Theorem~\ref{thm:general}, this is equivalent to $\langle \vu_I, \vv_J\rangle = 0$ unless either 1) $J = \varnothing$, or 2) $J=\{i\}$ and $I = \varnothing$ or $I=\{i\}$. This immediately yields the second condition in the statement of the Proposition. For the first one, we note that if $|J| \ge 2$, then $\langle \vu, \vv_J\rangle = \sum_{I} \langle \vu_I, \vv_J\rangle=0$, which implies $\langle \vu, \tilde \vv \rangle = 0$. Similarly, if $|I| \ge 2$, then $\langle \vu_I, \vv - \vv_{\varnothing} \rangle = \sum_{J \ne \varnothing} \langle \vu_I, \vv_J\rangle=0$ which implies $\langle \tilde \vu, \vv - \vv_{\varnothing} \rangle = 0$. This yields the first condition.

Conversely, if the two conditions in the statement hold, then the model is equivalent to one in which $\tilde \vu = 0$ and $\tilde \vv = 0$. Together with $\langle \vu_i, \vv_j \rangle = 0$ unless $i=j$, this means that the conditional relations~\eqref{eq:causal-independence} hold.
\end{proof}

\section{Exponential families}
\label{sec:exponential}

We briefly elaborate on the connection between interaction decompositions of embeddings and classical ideas related to exponential families and graphical models. Our discussion on these topics follows~\cite{ayInformationGeometry2017a}.

Let $\mathcal Z$ be a finite set and let $\mathcal P_+(\mathcal Z)$ be the space of probabilities over $\mathcal Z$ with full support. A (centered) \emph{exponential family} is a subset of $\mathcal P_+(\mathcal Z)$ of the form
\begin{equation}\label{eq:exponential-family}
\mathcal E(\mathcal L) = \left\{\left(\frac{e^{f(z)}}{\sum_{z'} e^{f(z')}}\right)_{z \in \mathcal Z} \colon f \in \mathcal L\right\},
\end{equation}
where $\mathcal L \subset \RR^{\mathcal Z}$ is a vector space of real-valued functions. In particular, assume that $\mathcal Z = \Z_1 \times \ldots \times \Z_k$ and let $\mathcal S$ be a (non-empty) family of subsets of $[k]$. The \emph{hierarchical model} associated with $\mathcal S$ is the exponential family defined by the space of functions
\begin{equation}\label{eq:hierarchical}
\mathcal L_{\mathcal S} := \bigoplus_{I \in \mathcal S} E_{I},
\end{equation}
where $E_I$ are the interaction spaces for $\RR^\Z$, defined as in Proposition~\ref{prop:projections}.
Thus, $\mathcal F_{\mathcal S}$ is the space of functions of the form $f(z) = \sum_{I \in \mathcal S} g_I(z_I)$ with $g_I: \Z_I \rightarrow V$. 
In particular, when $\mathcal S$ is the sets of cliques of a graph $G$ with $k$ nodes, then $\mathcal E(\mathcal F_{\mathcal S})$ is called a \emph{graphical model} for $G$. The {Hammersley-Clifford Theorem}~\cite[Theorem 2.9]{ayInformationGeometry2017a} relates the structure of the graph $G$ --- and the corresponding decomposition in~\eqref{eq:hierarchical} --- with conditional independence conditions between factors.

In this work, we assume that $\mathcal Z = \mathcal X_1 \times \ldots \times \X_m \times \Y_1 \times \ldots \times \Y_n$, so $\Z$ is a product of two groups of factors, $\X = \prod_{i=1}^m \X_i$ and $\Y = \prod_{i=1}^n \Y_i$. Moreover, we consider functions $f(z)$ in~\eqref{eq:exponential-family} of the form $f(z) = \langle \vu(x), \vv(y) \rangle$ where $z=(x,y)$ and $\vu: \X \rightarrow V$ and $\vv: \Y \rightarrow V$ are embeddings into a vector space.\footnote{The set of all functions of the form $f(z) = \langle \vu_f(x), \vv_f(y) \rangle$ do not form a vector space if $\dim V < \min(|\X|, |\Y|)$; however we can always consider the space generated by a collection of such functions.} In the proof of Theorem~\ref{thm:general}, we show that if the interaction decomposition of a general $f = \langle \vu(x), \vv(y) \rangle$ is
\begin{equation}\label{eq:interaction-decomposition-f}
f(z) =  \sum_{I \subset [m], J \subset [n]} g_{I \sqcup J}(z),
\end{equation}
then we have that
\[
g_{I\sqcup J}(z) = \langle \vu_{I}(x), \vv_{J}(y) \rangle, 
\]
where $z=(x,y)$  and $\vu_I$ and $\vv_J$ are terms in the interaction decompositions of the embeddings $\vu$ and $\vv$.
Thus, the vanishing of terms in~\eqref{eq:interaction-decomposition-f} corresponds to orthogonality conditions on the interaction components of the embeddings.

\section{Examples of geometric structures}
\label{sec:geometry-examples}

We noted in the main body of the paper that as more interaction components of an embedding $\vw: \Z \rightarrow V$ vanish, the structure of the points $\{\vw(z)\colon z \in \Z\}$ becomes more regular. We illustrate this by visualizing a few small examples.

If $\Z = \Z_1 \times \Z_2$ with $|\Z_i| = 2$, then $\vw_{\{12\}}=0$ means that $\{\vw(z): z \in \Z\}$ are vertices of a parallelogram in an affine plane, as discussed in~Example~\ref{ex:analogies}. If $\vw_{\{12\}}$ is non-zero, then the four points can be in general position, \ie, the vertices of a three-dimensional simplex (Figure~\ref{fig:geometric-examples}, first row).

If $\Z = \Z_1 \times \Z_2$ with $|\Z_1|=2$ and $|\Z_2|=3$, then $\vw_{\{12\}}=0$ means that $\{\vw(z): z \in \Z\}$ are vertices of a (three-dimensional) triangular prism. If $\vw_{\{12\}}$ is non-zero, then the six points can be in general position (Figure~\ref{fig:geometric-examples}, second row; note that in the second case we project the five-dimensional structure to 3D).

If $\Z = \Z_1 \times \Z_2 \times \Z_3$, $|\Z_i|=2$, and all order two interaction components are zero, then $\{\vw(z): z \in \Z\}$ are vertices of a (three-dimensional) parallelepiped. 
When only $\vw_{\{13\}} = 0$, the embeddings of $(z_1, z_2, z_3)$ for any fixed value of $z_2$ form vertices of a planar parallelogram, but these parallegrams are not parallel (Figure~\ref{fig:geometric-examples}, third row, again projecting to 3D).

\vspace{.1cm}

\begin{figure}[h]
\centering
\includegraphics[width=0.18\textwidth]{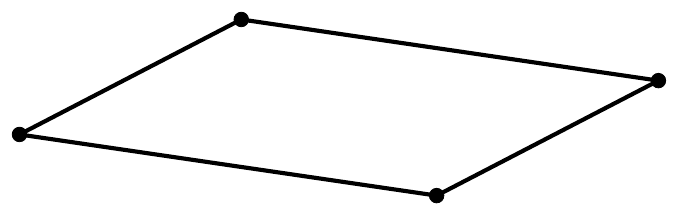} \qquad
\includegraphics[width=0.18\textwidth]{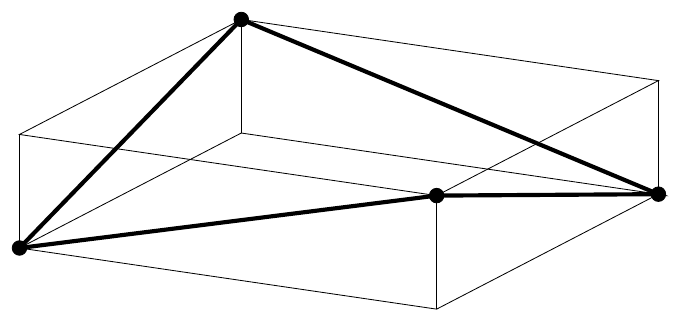}\\[.5cm]
\includegraphics[width=0.13\textwidth]{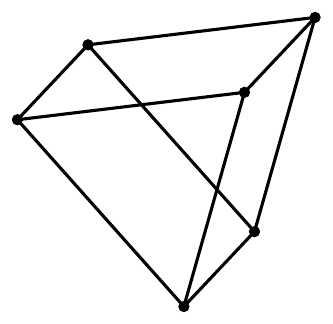} \qquad \qquad
\includegraphics[width=0.13\textwidth]{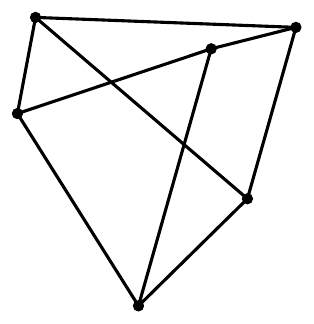}\\[.1cm]
\includegraphics[width=0.12\textwidth]{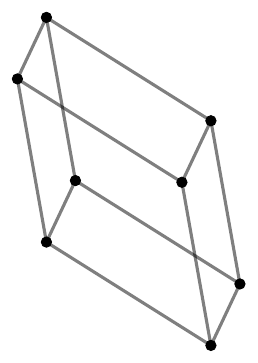} \qquad \qquad
\includegraphics[width=0.12\textwidth]{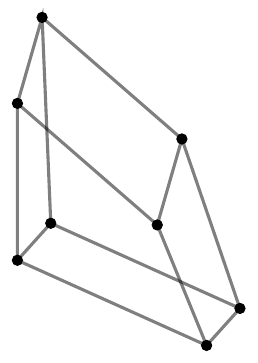}\\
\caption{
\small
\emph{Top:} two factors  without (left) and with (right) pairwise interactions. \emph{Middle:} two factors with two and three elements without (left) and with (right) pairwise interactions. \emph{Bottom:} three binary factors without pairwise interactions (left) and with only two out of three pairwise interactions (right). Note that in the last case the top and bottom faces are still parallelograms, but are not parallel.}
\label{fig:geometric-examples}
\end{figure}

\section{Plot details}
\label{sec:details}

\textbf{Figure~\ref{fig:qualitative-examples}:} \emph{Left:} We use ST5-XL~\cite{niSentenceT5ScalableSentence2021} to embed the attributes ``big,'' ``small,'' ``new,'' ``old,'' the objects ``bike,'' ``car,'' ``boat'', and all of their combinations, and plot the embeddings in 3D with PCA (adding colored lines for visualization).
\emph{Right:} We consider $7 \times 7$ pairs of words (listed in the plot), compute the embedding decompositon for the entire set $\vu = \vu_0 + \vu_1 + \vu_2 + \vu_{12}$ using Proposition~\ref{prop:projections} and then show the norm of the interaction component $\vu_{12}$ for each pair.

\textbf{Figure~\ref{fig:train-interactions}:} We use the MinGPT codebase\footnote{\url{https://github.com/karpathy/minGPT/tree/master}} and use the default model ``GPT-mini'' with vocab size $10$ (total 2.7M parameters). We generate categorical distributions over $[10]\times[10]\times [10]$ with the desired factorization properties and create a dataset which samples from this distribution. We train for 20K steps with batch size 512, optimizer AdamW with learning rate 3e-4, and other default parameters from the codebase. We use a callback to compute at every training step the embeddings of all possible 10$\times$10 inputs and their interaction components. We project the embeddings and the components onto the space spanned by centered output embeddings (which is the space spanned by differences $\vv(z)- \vv(z')$). In the top plot we show the mean of the ratio between the (projected) embedding norm and the interaction component norms. In the bottom plot, we show embeddings projected in 3D with PCA for a random quadruple $\{z_1,z_1'\} \times \{z_2,z_2'\}$ at the beginnning and at the end of training. For the second plot, we permute inputs before feeding them to the transformer so that the tokenization cannot help the model, however we compute the interactions for the original (latent) factorization to reflect the structure of the data distribution.

\end{document}